\documentclass[conference,letterpaper]{IEEEtran}

\usepackage[utf8]{inputenc} 
\usepackage[T1]{fontenc}
\usepackage{url}
\usepackage{ifthen}
\usepackage{cite}
\usepackage[cmex10]{amsmath} 

\usepackage{mathtools}

\newtheorem{theorem}{Theorem}[subsection]

\newtheorem{lemma}{Lemma}
\newtheorem{proof}{Proof}

\newtheorem{definition}{Definition}
\newtheorem{assumption}{Assumption}

\DeclareMathOperator*{\argmin}{arg\,min}
\newcommand{\X}{\mathbf{X}}
\newcommand{\Ropt}{R_{f_{{opt}}}}

\newcommand{\Ftheta}{\{f_{\theta},\theta \in \Theta\}}
\newcommand{\topt}{\theta_{opt}}

\newcommand{\gemp}{\hat{\theta}^g_n}

\newcommand{\optemp}{\hat{\theta}^{\fopt}_n}
\newcommand{\Pro}[1]{\Pr\left( #1 \right)}
\newcommand{\fopt}{f_{{opt}}}

\DeclarePairedDelimiter\floor{\lfloor}{\rfloor}

\DeclarePairedDelimiterX{\infdivx}[2]{(}{)}{%
  #1\;\delimsize\|\;#2%
}
\newcommand{\infdiv}{D_{KL}\infdivx}

\newcommand{\infdivv}{\mathcal{D}_{KL}\infdivx}

\begin{document}

\title{Error Exponent in Agnostic PAC Learning} 


\author{%
  \IEEEauthorblockN{Adi Hendel and Meir Feder}
  \IEEEauthorblockA{School of Electrical Engineering\\ 
                    Tel-Aviv University, Tel-Aviv, Israel\\
                    Email: adihendel@mail.tau.ac.il, meir@tau.ac.il}
}


\maketitle


\begin{abstract}
    Statistical learning theory and the Probably Approximately Correct (PAC) criterion are the common approach to mathematical learning theory. PAC is widely used to analyze learning problems and algorithms, and have been studied thoroughly. Uniform worst case bounds on the convergence rate have been well established using, e.g., VC theory or Radamacher complexity. However, in a typical scenario the performance could be much better. In this paper, we consider PAC learning using a somewhat different tradeoff, the error exponent - a well established analysis method in Information Theory - which describes the exponential behavior of the probability that the risk will exceed a certain threshold as function of the sample size. We focus on binary classification and find, under some stability assumptions, an improved distribution dependent error exponent for a wide range of problems, establishing the exponential behavior of the PAC error probability in agnostic learning. Interestingly, under these assumptions, agnostic learning may have the same error exponent as realizable learning. The error exponent criterion can be applied to analyze knowledge distillation, a problem that so far lacks a theoretical analysis.
\end{abstract}

\section{Introduction}
\label{sec:Introduction}
Statistical machine learning studies the generalization ability and convergence rate of learning algorithms. One of the most popular criteria for learnability is the Probably Approximately Correct (PAC) criterion, suggested in \cite{Valiant84,PAC}, which describes the probability of a learning algorithm to output a hypothesis that is not too far from the optimal one. 

In this work, we will consider the class of Empirical Risk Minimization (ERM) predictors, which is the most prominent method for learning problems. ERM predictors choose the hypothesis achieving minimal loss on a given training sample, and their analysis under the PAC criterion is well established through VC theory \cite{StatisticalLearningTheory,vapnik1974theory}. 

Classical setting divides the learning problem into two cases - Realizable learning, in which the target function is taken from the hypothesis class, and Agnostic learning, in which the target function could be outside the class. The general worst-case upper bounds of both cases are well established, see \cite{bousquet2004introduction} for example.

Although VC theory is powerful, it provides a uniform upper-bound for the worst case scenario, where in a typical scenario the convergence rate could be much faster, as suggested by \cite{bousquet2021theory}. Actually, the recent rise of deep learning demonstrate that uniform bounds fail to describe many practical situations and better characteristics comes from considering non-uniform, possibly distribution dependent analysis. 

In this paper we consider agnostic PAC learning for the case of binary labels and 0-1 loss function. We derive an improved distribution-dependent error exponent for the PAC error probability, using some assumptions, for a wide range of learning problems. Moreover, we show that under the specified assumptions, the derived error exponent can be the same for both agnostic and realizable learning.

\subsection{Related Work}
\label{section: Related Work}
 VC theory and the PAC model provide conditions for uniform consistency and bounds, that are achieved, e.g., by ERM predictors \cite{StatisticalLearningTheory}. This theory fails to explain the success of recent learning models, such as neural networks, as presented in \cite{cohn1990can,cohn1992tight}, where practical learning rates can be much faster than the ones predicted by the VC theory. Moreover, in \cite{nagarajan2019uniform} different types of over-parameterized models are analyzed and it is proved that any uniform bound would yield a bad generalization bound. This issue motivated theories that provide better, non-uniform learning rates. 
 
 In this direction, works such as \cite{nitanda2019stochastic,haussler1994bounds,audibert2007fast},
 establish improved bounds for specific cases and algorithms. However, these works do not provide a general theory. \cite{lever2010distribution,lever2013tighter} developed tighter bounds for distribution dependent PAC-Bayes priors. \cite{bendavid2014sample} showed the existence of classes with faster rates than the classical agnostic bound, but the provided condition for such a rate is impractical for infinite feature spaces. Other works relax the uniformity property, as done by \cite{benedek1988nonuniform}, who proposed a relaxed model of PAC in which the bound on the learning rate may depend on a hypothesis, but is uniform on all distributions consistent with that hypothesis.
 Other works focus on totally non-uniform learning bounds. For example, \cite{hanneke2020universal,hanneke2021learning} established a theory for non-uniform consistency, in which an algorithm is considered consistent if it convergence to the optimal risk for any ground truth, and showed there exists such algorithm for separable metric spaces.
 In \cite{bousquet2021theory} a theory for non-uniform PAC learning in the realizable setting is developed showing that the learning rate can be one of 3 types: exponential, linear and arbitrarily slow. 
 
\section{Problem Formulation}
\label{sec:Problem Formulation}
Let the training data be $n$ pairs of data samples and their labels $(x_1, y_1),...,(x_n, y_n)$, where $x_i$ are i.i.d and drawn from a feature space $\X\subseteq\mathbf{R^N}$ according to an unknown distribution $\mathcal{F}$, and the labels $y_i \in \{0,1\}$ are generated by some unknown deterministic function $y_i = g(x_i)$, called the ground-truth function. We have a hypothesis class $F_{\Theta}=\{f_\theta, \theta \in \Theta\}$, and would like to find the closest hypothesis to the ground truth in the class under some loss function.
We focus on the setting in which the hypotheses range is binary and the loss function is the 0-1 loss:
\[
    \ell(y_a, y_b) = \begin{cases}
       0 &\quad y_a = y_b\\
       1 &\quad y_a \neq y_b
     \end{cases}
\]
The following notations are with regard to some arbitrary function $f$, where $f$ can be the ground truth or some other function in discussion.
For hypothesis class $F_{\Theta}$, denote the risk between hypothesis $f_{\theta} \in F_{\Theta}$ and some function $f(x)$ as:
\begin{align}
    R_f(\theta)= R(f,f_{\theta})= \int_{\X}\ell(f(x),f_{\theta}(x))d\mathcal{F}(x) \quad ,\theta \in \Theta
\end{align}
and the empirical risk between $f_{\theta}$ and $f(x)$ on sample $x^n$ as:
\begin{align}
    R_f^{emp}(\theta, x^n) = \frac{1}{n}\sum_{i=1}^n \ell(f(x_i),f_{\theta}(x_i))  \quad, \theta \in \Theta
\end{align}
In this paper, we analyze the Empirical Risk Minimization (ERM) algorithm, which selects the hypothesis that minimizes the empirical risk on the sample $x^n$ out of all the hypotheses in the class. Specifically, the ERM on a sample $x^n$ and hypothesis class $\Theta$, with regard to a function $f(x)$, is defined as: 
\begin{align}
    \hat{\theta}^f_n = \hat{\theta}^f(x^n) = \argmin_{\theta \in \Theta}R_f^{emp}(\theta,x^n)
\end{align}
Whenever there are multiple hypotheses with the same minimal empirical risk, we use the convention of choosing the one maximizing the true risk (i.e., the worst one).

Denote the hypothesis achieving minimum risk with regard to the ground truth $g$ as $\topt$:
\begin{align}
    \topt = \argmin_{\theta \in \Theta}R_g(\theta), \quad
    \fopt = f_{\topt}
\end{align}
We will refer to $\fopt$ as the projection of $g$ on $F_{\Theta}$.
We assume for simplicity that the hypothesis class is non-degenerate in the sense that there is no subset of the feature space $X^{'} \subseteq \X$ for which all hypotheses coincide. That is,  for any set $\X^{'} \subseteq \X$ with positive probability, we have:
\begin{align}\label{degenerate condition}
    \Pro{f_{\theta}(x) = \text{const} \ \forall  \theta \in \Theta \mid x \in X^{'}} = 0 
\end{align}
This is not restrictive as any part of the feature space on which all hypotheses coincide will contribute the same risk to all hypotheses, thus not affecting the choice of ERM. We'll also assume that for any positive (Lebesgue) measure set the probability measure is positive (this is non restrictive as such regions with zero probability have no effect).
\subsection{PAC Learning}
In the context of ERM, we say that the class $F_\Theta$ is (agnostic) PAC learnable \cite{shalev2014understanding} if there exist a sample size $N(\epsilon,\delta)$ and an algorithm $\hat\theta_n^{g}$ such that for every ground truth function $g(x)$, every probability distribution $\mathcal{F}$ on $\X$ and every $\delta,\eta \in (0,1)$, for $n>N$, with probability at least $1-\eta$ we have:
\begin{align}
    R_g(\hat\theta_n^{g}) < R_g(\theta_{opt}) + \delta.
\end{align}
PAC actually describes the relationship between three quantities: the deviation from the optimal risk $\delta$, the probability $\eta$ for deviation larger than $\delta$ and the size of the sample $n$. We will refer to $\eta$ as the PAC error probability for shortness.

The analysis of learning algorithms using PAC is usually done by writing one parameter as a function of the other two. Most notably, writing $n$ as a function of $\eta$ and $\delta$ (known as sample complexity) or writing $\delta$ as function of $n$ for some fixed value of $\eta$ (known as excess risk). In this way, we can say one algorithm is better than the other if, for example, it has a better sample complexity (i.e., $n$ increases slower as a function of $\frac{1}{\delta}$ for a fixed $\eta$). We propose to fix $\delta$ and to look instead at the probability of deviation $\eta$ as function of the sample size $n$. In this case, we say that one algorithm is better than the other if $\eta$ decays faster as a function of $n$.

\subsection{VC Theory}
VC theory \cite{StatisticalLearningTheory} provides consistency conditions and uniform (worst-case) bounds for PAC learning of ERM predictors. This is done using the VC dimension of the hypothesis class, denoted $h$, defined as the maximum sample size $h$ for which the sample can be separated into two classes in all $2^h$ possible label sequences, using functions from the hypothesis class.

VC theory provides the following well known results for ERM (see \cite{bousquet2004introduction} for example). In agnostic learning, for $n>h$, with probability at least $1-\eta$ we have:
\begin{equation}\label{VC agnostic2}
\begin{aligned}
    R_g(\gemp)-R_g(\topt) \leq 4\sqrt{2\frac{h\ln{\frac{2en}{h}} + \ln{\frac{2}{\eta}}}{n}}
\end{aligned}
\end{equation}
In realizable learning, for $n>h$, with probability at least $1-\eta$:
\begin{equation}\label{VC realizable}
\begin{aligned}
    R_g(\gemp) \leq 4\frac{h(\ln{\frac{2ne}{h}}) - \ln{\frac{\eta}{4}}}{n}
\end{aligned}
\end{equation}
These bounds describe the generalization and convergence of the ERM predictor.

Focusing on (\ref{VC agnostic2}), we can get $\eta$ as a function of $\delta$ and $n$ by setting the right hand side to $\delta$:
\begin{align}\label{agnostic bound}
    \Pro{R_g(\gemp) - R_g(\theta_{opt})> \delta} \leq 2e^{-\frac{\delta^2}{32}n + h\ln{\frac{2en}{h}}}
\end{align}
Similarly, we get the following for the realizable case:
\begin{align}\label{realizable bound}
    \Pro{R_g(\gemp) - R_g(\theta_{opt})> \delta} \leq 4e^{-\frac{\delta}{4}n + h\ln{\frac{2en}{h}}}
\end{align}
This formulation of the upper bound shows that the PAC error probability $\eta$ decays exponentially with $n$ and allows us to explore its error exponent. Recall that the error exponent $d$ of a series $a_n$ is defined as:
\begin{align}
    d = -\lim_{n \to \infty} \frac{1}{n}\ln{a_n}
\end{align}
We will use the notation $a_n \doteq b_n$ to indicate that series $a_n$ has the same error exponent as $b_n$. The concept of error exponent (see section 5.6 in \cite{gallager1968information} for example), was proven useful in Information Theory for analyzing the decay rate of probabilities to zero. It allows utilizing powerful mathematical tools such as the method of types \cite{mot} and Sanov's theorem \cite{sanov1961probability}. We can see from (\ref{agnostic bound}) that the error exponent in the agnostic case is $\frac{\delta^2}{32}$, and from (\ref{realizable bound}) that the error exponent in the realizable case is $\frac{\delta}{4}$. We note that the bound in (\ref{VC agnostic2}) can be manipulated using a chaining
technique \cite{anthony1999neural} to get rid of the $\ln{\frac{2en}{h}}$ factor, but the resulting error exponent will be worse. 
In the next sections, we will derive an improved distribution-dependent bound for the PAC error probability
   $\eta= \Pro{R_g(\gemp) - R_g(\theta_{opt})> \delta}$ for the agnostic case.
This will be done using some assumption on the learning problem (i.e., on the hypothesis class and the ground truth) described in the next sections. Under these assumptions the error exponent in the agnostic case can be the same as in the realizable case for small enough $\delta$.

\section{Preliminaries}\label{section: Preliminaries}
In this section we introduce a few key concepts. In order to provide some intuition, we will use the k-boundary hypothesis class as a case study and demonstrate these concepts on it.
\label{sec:Definition}
\begin{definition}[k-boundary hypothesis class]\label{def k-boundary hypothesis class}
Let $X \subseteq \mathbf{R}$. The k-boundary hypothesis set is defined as
\begin{align*}
f_{b_1,...,b_k}(x) = \begin{cases}
       0 &\quad x < b_1\\
       1 &\quad b_1 \leq x < b_2 \\
       0 &\quad b_2 \leq x < b_3 \\
       ... \\
       1 &\quad b_k \leq x  \\
     \end{cases}
\end{align*}
Where $b_1 \leq b_2 \leq ... \leq b_k, \quad x,b_1,...,b_k\in X $. For uniqueness, equality is allowed only between the first 2 parameters or between last parameters (e.g., $b_1 = b_2 < b_3 < ...< b_{k-2}=b_{k-1}= b_k$).
\end{definition}

For example, on the feature space $X = [0,1]$ with uniform distribution, the 2-boundary function with parameters $b_1 = 0.5, b_2 = 0.9$ is
\begin{align*}
    g(x) = f_{b_1,b_2}(x) = \begin{cases}
       0 &\quad 0 \leq x < 0.5\\
       1 &\quad 0.5 \leq x < 0.9 \\
       0 &\quad 0.9 \leq x \leq 1 \\
     \end{cases}
\end{align*}
We will use this example throughout this section to demonstrate the presented concepts.
Another important hypothesis class, which is more closely related to neural networks, is the class of linear classifiers:
\begin{definition}[linear hypothesis class]\label{def: linear hypothesis class}
Let there be a feature space $X \subseteq \mathbf{R^k}$. The k-dimensional linear hypothesis set is:
\begin{align*}
f_{b_0,...,b_{k}}(x) = \mathbf{1}(b_0 +b_1 x_1 + ... + b_{k}x_k > 0 ) 
\end{align*}
Where $(b_0,...,b_k) \in \mathbf{R}^{k+1}$, $(x_1,...,x_k) \in X$ and $\mathbf{1}\{.\}$ is the indicator function.
\end{definition}
\begin{definition}[Generalized Optimum Point]\label{def:glp}
For hypothesis class $\Ftheta$ and ground truth function $g(x)$, we say that $\theta \in \Theta$ is a generalized optimum point (GLP) of \(R_g(\theta)\) if $\forall  \Tilde{\theta}\in \Theta\setminus\theta$ there exists a set $\Tilde{X} \subseteq \X$ with positive probability, such that $\forall x \in \Tilde{X}$ we have $\ell(f_\theta(x),g(x)) < \ell(f_{\Tilde{\theta}}(x),g(x))$.
\end{definition}

In simple words, $\theta$ is a GLP if no other hypothesis can beat it uniformly on the feature space $\X$. Notice that the hypothesis $\topt$ minimizing the risk is always a GLP, as for every other hypothesis in the class there must exist a set for which $\topt$ is uniformly better, otherwise it would not be the minimizer of the risk. We will refer to $\topt$ as the global optimum. 

Consider for example the 1-boundary hypothesis class with a ground truth $g$ as described above. The GLP's will be $\theta_0 = 0.5$ and $\theta_1 = 1$, as no other hypothesis $\theta \in [0,1]$ achieves a lower loss for all $x \in \X$. Notice that these are the only GLP's since any other hypothesis $\theta$ is no better (for all $x$) than either $\theta_0$ or $\theta_1$. This divides the parameter space into two groups: hypotheses that are no better than $\theta_0$ and hypotheses that are no better than $\theta_1$.
We can informally say that when an ERM learns from $g$ using the 1-boundary hypothesis class, there is going to be a competition between these 2 groups. The following definition generalizes this concept. 

For each GLP $\theta^{*}$, denote the set $A_{\theta^{*}} \subseteq \Theta$:
\begin{definition}[$A_{\theta}$ region]\label{def:A_region}
Let $\Theta_{opt}$ be the set of GLP's of $R_g(\theta)$ and $\theta_0 \in \Theta_{opt}$ be the global optimum. For every $\theta^* \in \Theta_{opt}$ denote the regions:
\begin{equation}
\begin{aligned}
    \Tilde{A}_{\theta^*} =
    \{\theta \in \Theta \mid \Pro{\ell\big(f_{\theta^*}(x), g(x)\big) \leq \ell\big(f_{\theta}(x), g(x)\big)}=1  \}
\end{aligned}
\end{equation}
In order to make these regions disjoint, we handle the intersections in the following way:
\begin{enumerate}
    \item remove all overlaps from $\Tilde{A}_{\theta_0}$:\\
    \[A_{\theta_0} = \Tilde{A}_{\theta_0} \setminus \bigcup_{\theta^* \in \Theta_{opt} ,\ \theta^* \neq \theta_0}A_{\theta^*}\]
    \item For the other regions $\Tilde{A}_{\theta^*}, \theta^* \neq \theta_0$, arbitrarily assign the intersection to one of the regions such that there will not be any overlap, to obtain the regions $A_{\theta^*}$
\end{enumerate}
\end{definition}
These regions form a complete partitioning of $\Theta$ such that $\cup_{\theta^* \in \Theta_{opt}} A_{\theta^*} = \Theta$
(see proof in appendix \ref{appendix perlimineries}).

In simple words, $A_{\theta^*}$ is the set of hypotheses in $\Theta$ that are no better than the GLP $\theta^*$ for any given $x \in X$ (with probability 1). For the example above we have the sets $A_{\theta_0} = (0,0.9)$ and $A_{\theta_1} = (0.9,1)$.

Note that in any (non-degenerate) agnostic learning problem we will have at least 2 GLP's, because if $g$ is outside the class, there must be a set in $\X$ with positive probability for which $g$ is different than $\fopt$. Any hypothesis equal to $g$ on this set will be universally better than $\fopt$ on this set, and will not belong to 
$A_{\theta_0}$ (this is a consequence of (\ref{degenerate condition})). Thus, there must be other GLP's in addition to $\theta_0$. 

\begin{definition}[Dominating region]\label{def:D_region}
For hypothesis class $\Ftheta$, the Dominating region of $\theta_a$ on $\theta_b$, where $\theta_a, \theta_b \in \Theta$, with regard to $g(x)$, denoted as $D(\theta_a, \theta_b) \subseteq \X$, is defined as
\begin{equation}
\begin{aligned}
&D(\theta_a, \theta_b) = \\
&\big\{x \in \X \mid \Pro{\ell\big(f_{\theta_a}(x), g(x)\big) < \ell\big(f_{\theta_b}(x), g(x)\big)} = 1\big\}
\end{aligned}
\end{equation}

\end{definition}
The dominating region is the set in the feature space for which $\theta_a$ achieves lower loss than $\theta_b$ (i.e., $f_{\theta_a} = g$, $f_{\theta_b} \neq g$). Using our example, the dominating region of $\theta_0$ on $\theta_1$ is $D(\theta_0, \theta_1) = (0.5, 0.9)$.

\begin{definition}[Stability]\label{def:stability}
We say that a GLP $\theta^*$ is stable if we can define a distance in $\Theta$ and there exist $\epsilon > 0$ such that for every $\theta$ with distance $||\theta - \theta^*|| < \epsilon$ the following holds with probability 1:
\begin{equation}
\begin{aligned}
\ell(f_{\theta^*}(x),g(x)) \leq \ell(f_{\theta}(x),g(x)) 
\end{aligned}
\end{equation}
\end{definition}
Informally, $\theta^*$ is a stable GLP if any hypothesis in its neighborhood does not have an improved classification ability with regard to any $x \in \X$. Using our example from above, both $\theta_0$ and $\theta_1$ are stable.
\section{Theoretical Results}\label{section: Theoretical Results}
We consider the following assumptions.

\begin{assumption}\label{assumption stable}
$f_{opt}$ is a stable GLP.
\end{assumption}
\begin{assumption}\label{assumption unique}
$f_{opt}$ is a unique GLP.
\end{assumption}
\begin{assumption}\label{assumption consistent}
    The following is true in probability:
     \[ \lim_{n \to \infty} \hat{\theta}^g_n = \topt \]
\end{assumption}
\begin{assumption}\label{assumption complete}
$F_\Theta$ is a complete space. i.e., the limit of any Cauchy sequence in $\Theta$ is also in $\Theta$, and the limit of any sequence $f_{\theta_m} \in F_{\Theta}$, such that the series $\int_{\X}\ell(f_{\theta_m}(x),f_{\theta_{m+1}}(x))d\mathcal{F}(x)$ has a limit, is also in $F_{\Theta}$.
\end{assumption}
\begin{assumption}\label{assumption delta}
$0 < \delta < \delta_{max}$, 
where $\delta_{max}$ is denoted as:
\begin{align}\label{delta_max}
    \delta_{max} = \min\{\min_{\theta \notin A_{\theta_{0}}}\Ropt(\theta) ,\min_{\theta \notin A_{\theta_{0}}}R_g(\theta) - R_g(\theta_{opt})\}
\end{align}
\end{assumption}
Assumption \ref{assumption unique} is needed mainly to ease the analysis and can be generalized. Relaxing this assumption will cause the ERM to alternate between multiple 
 $A_\theta$ regions of the (non-unique) global GLP's. Assumption \ref{assumption consistent} is a non-uniform consistency requirement (this is weaker than finite VC dimension), which is reasonable. Assumption \ref{assumption complete} is mainly a mathematical technicality. Assumption \ref{assumption delta} means that we are looking at what happens for small enough $\delta$, which is reasonable as we are interested in the behavior in the asymptotic regime. Assumption \ref{assumption stable} is the only one that poses a significant constraint. Nevertheless, a wide range of learning problems satisfy it, such as $k$-boundary class with a ground truth with finite number of transition points (see proof in appendix \ref{appendix: K-boundary} ). It is also satisfied by some cases of linear classifiers, which are more closely related to neural networks. 
An example of a ground truth that satisfies these assumptions, using a 2-dimensional linear hypothesis set is shown in Figure \ref{fig:linear} where the optimal linear hypothesis is $\mathbf{1}(x_2 > 1.3-x_1)$, and any change in its parameters will result in mis-classification of more features, thus it is a stable GLP.
\begin{figure}[htbp]
\centering
\includegraphics[width=0.4\textwidth]{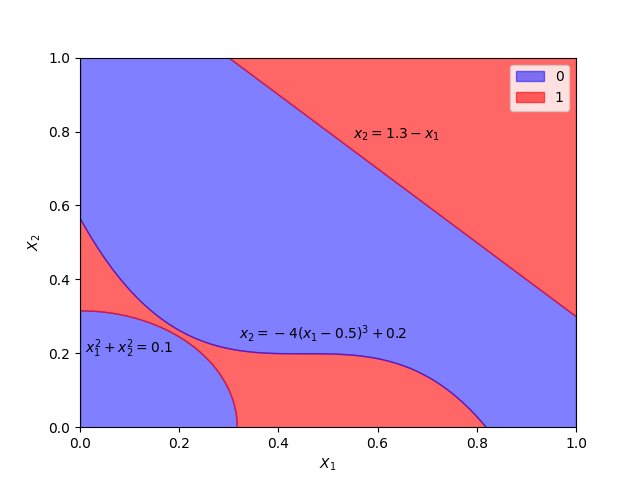}
\caption{Ground truth with stable optimal 2-dimensional linear hypothesis (see definition \ref{def: linear hypothesis class}). $x_1$ and $x_2$ are uniformly distributed in $[0,1]$. The optimal hypothesis is achieved by $b_0 =-1.3,\ b_1=1,\ b_2=1$.}
\label{fig:linear}
\end{figure}

We can now move to state our main results.
\begin{theorem}\label{theorem: realizable+agnostic}
    Given a hypothesis class \(\Ftheta\), and ground truth function $g$ with projection $\fopt$ on the hypothesis class, the following holds under assumptions \ref{assumption stable}-\ref{assumption delta}:
    \begin{equation}\label{realizable+agnostic}
    \begin{aligned}
        &\Pro{R_g(\gemp) - R_g(\theta_{opt})> \delta} = \\
        &P_R +(1-P_R)\Pro{\gemp \notin A_{\theta_0} \mid \Ropt(\optemp) < \delta}
    \end{aligned}
    \end{equation}
\end{theorem}
where $P_R=\Pro{\Ropt(\optemp) > \delta}$ is the realizable PAC error probability when learning from $\fopt$ (see proof in appendix \ref{appendix: proof 1}).
This theorem decomposes the PAC error probability into the error incurred in realizable learning and the additional error incurred in agnostic learning. Notice that for realizable learning, $g \in F_{\Theta}$, we get $\Pro{\Ropt(\gemp) > \delta} = P_R$, as expected.

Denote the KL divergence projection of a distribution $\Tilde{Q}$ on a set of distributions $\Tilde{\Pi}$ as:
\begin{align}
    \infdiv{\Tilde{\Pi}}{\Tilde{Q}} = \inf_{P \in \Tilde{\Pi}} \infdivv{P}{\Tilde{Q}} 
\end{align}
where $\infdivv{P}{\Tilde{Q}}$ is the KL divergence.
\begin{theorem}\label{theorem: error exponent}
Under assumptions \ref{assumption stable} - \ref{assumption delta}, if $\Theta$ has a finite VC dimension, there exists a positive real number $d \in \mathbf{R}^+$ such that the following holds:
    \begin{align}
    \Pro{R_g(\gemp) - R_g(\theta_{opt})> \delta} \doteq e^{-n\cdot \min\{\frac{\delta}{4}, d\}}
\end{align}
where $d = \infdiv{\Pi}{Q}$, $\Pi$ is a set of distributions on some alphabet $\chi$, induced by the distribution on $\X$ for which the ERM will output a hypothesis outside of $A_{\theta_0}$ and $Q$ is the true distribution on the alphabet. 
\end{theorem}
The proof of the Theorem is provided in appendix \ref{appendix: proof theorem 2}. The distribution $Q$ and the set of distributions $\Pi$ will be explicitly derived in the next section.  

This theorem establishes the exponential behavior of the PAC error probability and is achieved by showing that the error exponent of $\Pro{\gemp \notin A_{\theta_0} \mid \Ropt(\optemp) < \delta}$ is $\infdiv{\Pi}{Q}$ and using the uniform realizable learning bound in (\ref{realizable bound}). This implies that any improved realizable bound can be plugged into theorem \ref{theorem: error exponent} to get an improved agnostic bound, and the requirement of finite VC dimension might be unnecessary. 

The achieved error exponent $\min(\frac{\delta}{4}, d)$ is better than the classical error exponent for agnostic learning in (\ref{agnostic bound}), which is $\frac{\delta^2}{32}$, as it is linear in $\delta$ instead of quadratic in $\delta$.

Notice that because $d$ is independent of $\delta$, for $\delta < \min\{4d, \delta_{max}\}$ the error exponent is $\frac{\delta}{4}$, which is the same as the worst case {\em realizable} learning exponent. Thus, not only the error exponent is much better than the general one for agnostic learning, it also shows that agnostic learning might be no harder than realizable learning in some cases. This result can be expressed as a bound on the excess risk:
\begin{align}
    \delta = O(\frac{1}{n}\ln{\frac{1}{\eta}})
\end{align}

\subsection{Derivation of Error Exponent}
\label{sec:Analysis of PAC Criterion}
In this section we provide details on how to construct
the set $\Pi$ and the distribution $Q$.
The derivation in this section is partial and is done under the assumption that there are K+1 GLP's. However, this is only to simplify the already complex derivation and is not a requirement (see appendices \ref{appendix: PAC analysis} and \ref{appendix:generalization} for more details).
$\topt$ must be one of the $K+1$ GLP's. Denote $\theta_0 = \theta_{opt}, \ f_{\theta_{0}} = \fopt$.
For each GLP $\theta_i\ , i=1,...,K$ of $R_g(\theta)$, denote the following regions:
\begin{equation}\label{D_pairs}
    \begin{aligned}
         D_{i} = D(\theta_{0},\theta_{i}), \
         D^{'}_{i} = D(\theta_{i}, \theta_{0})
    \end{aligned}
\end{equation}
$D_i$ is the region that supports choosing $\theta_0$ over $\theta_i$ and $D^{'}_i$ is the opposite. Denote $\#D$ as the number of samples that fall in region $D$. $\theta_0$ achieves lower empirical risk than $\theta_i$ if $\#D^{'}_i < \#D_i$.
Denote the disjointified regions of $D_i,D^{'}_i$:
\begin{equation}\label{X_region}
\begin{aligned}
    &X_{i_1,...,i_r}= \{\cap_{j \in \{i_1,...,i_r\} } D_{j}\}\symbol{92}\{\cup_{j \notin \{i_1,...,i_r\} } D_j\} \\
    &X^{'}_{i_1,...,i_r}= \{\cap_{j \in \{i_1,...,i_r\} } D^{'}_{j}\}\symbol{92}\{\cup_{j \notin \{i_1,...,i_r\} } D^{'}_j\} \\
    &X_c= X\symbol{92}\{\cup_{j}{D^{'}_j}\cup_{i}{D_i}\},
\end{aligned}
\end{equation}
where $r=1,..,K$ and $1 \le i_r\le K$. Notice these regions are non-intersecting and \(D_i = X_{i}\cup\{\cup_{i_2}X_{i,i_2}\}\cup .. \cup \{\cup_{i_2,..i_K}X_{i,i_2,..i_K}\}\). We can write the following equation:
\begin{equation}\label{eq: constraint matrix}
    \begin{aligned}
        \begin{pmatrix}
        \#D_1 - \#D^{'}_1\\
        ...\\
        \#D_K - \#D^{'}_K\\
        \end{pmatrix} = 
        A\begin{pmatrix}
        \#X_1\\
        ...\\
        \#X_{1,..,K}\\
        ...\\
        \#X^{'}_{1,..,K}
\end{pmatrix}
    \end{aligned}
\end{equation}
where A is a matrix.
Denote the alphabet $\chi$:
\begin{equation}\label{alphabet}
\begin{aligned}
    &\chi = \{X_1,..,X_{1,..,k}, X^{'}_1,..,X^{'}_{1,..,k}, X_c\} = \{a_1, ..., a_{|\chi|}\}\\ 
\end{aligned}
\end{equation}
Denote the probability mass function $Q$ on $\chi$ such that $ Q(a_i) = \Pro{x \in S_i}$, $a_i \in \chi$,
where $a_i$ is the i'th symbol in $\chi$ and $S_i$ is the region corresponding to it. 
Denote $\Pi$:
\begin{equation}\label{eq: pi}
\begin{aligned}
    &\Pi = \Big\{(p_1,..,p_{|\chi|}) \mathrel{\Big|} 
    \Big\{A\begin{pmatrix}
p_1\\
...\\
p_{|\chi|-1}
\end{pmatrix} < 0 \Big\}^c,\sum_{i=1}^{|\chi|}p_i=1, p_i\ge0\Big\}
\end{aligned}   
\end{equation}
where $\{\}^c$ is the complement set. This set represents all distributions on $\chi$ for which the ERM will output a hypothesis outside of $A_{\theta_0}$ (and thus, sub-optimal). Notice that $Q \notin \Pi$, as $Q$ is the true distribution on $\chi$, and under it the ERM must converge to the optimal hypothesis due to assumption \ref{assumption consistent}.

Using a method of types based analysis \cite{mot,csiszar2011information,sanov1961probability} we get the following exponential behavior (see appendix \ref{appendix: proof theorem 2} for details):
\begin{equation}
\begin{aligned}
    \Pro{\gemp \notin A_{\theta_0} \mid \Ropt(\optemp) < \delta} 
     \doteq2^{-n\infdiv{\Pi}{Q}}
\end{aligned}    
\end{equation}
Furthermore, by using Theorem \ref{theorem: realizable+agnostic} and the worst case error exponent of realizable learning in (\ref{realizable bound}), we get Theorem \ref{theorem: error exponent}.

\section{Example}
\label{sec:Experiments}
This example shows how to compute the error exponent and that it empirically converges to the value in the theorem. Let  $\X = [0,1]$ with uniform distribution, and the ground truth is:
\begin{align*}
g(x) = \begin{cases}
       0 &\quad 0 < x < 0.6\\
       1 &\quad 0.6 < x < 0.9 \\
       0 &\quad 0.9 < x < 1 \\
     \end{cases}
\end{align*}
We use the 1-boundary hypothesis class for ERM learning.
The optimal hypothesis minimizing the risk is:
\begin{align*}
f_{opt}(x) = \begin{cases}
       0 &\quad 0 < x < 0.6\\
       1 &\quad 0.6 < x < 1 \\
     \end{cases}
\end{align*}
There are two GLP's:  $\theta_0 = 0.6, \ \theta_1 = 1$.
Their $A_\theta$ regions are $A_{\theta_0} = (0,0.9) \ ,\ A_{\theta_1} = (0.9,1)$.
The pairs of regions $D_1, D^{'}_1$, as defined in (\ref{D_pairs}), are $D_1 = (0.6, 0.9) \ ,\ D^{'}_1 = (0.9, 1)$.
The disjointified regions (as in (\ref{X_region})) are $X_1 = (0.6, 0.9) \ ,\ X^{'}_1 = (0.9, 1) \ ,\ X_c = (0, 0.6)$.
We get an alphabet $\chi=\{X_1, X^{'}_1, X_c\}$ with probabilities  $Q = (0.3, 0.1, 0.6)$.
The region $\Pi$ is: 
\begin{equation*}
\begin{aligned}
    &\Pi = \Biggl\{(p_1,p_2,p_3) \mathrel{\Bigg|} 
    \begin{pmatrix}
    -1 & 1
    \end{pmatrix}
    \begin{pmatrix}
    p_1\\
    p_2
    \end{pmatrix} \ge0 , 
    \sum_{i=1}^3 p_i = 1, \ 0 \leq p_i 
    \Biggl\} 
\end{aligned}
\end{equation*}
Computing $\infdiv{\Pi}{Q}$ is a simple constraint optimization problem with solution $\infdiv{\Pi}{Q} = 0.0551$. The error exponent using bound (\ref{agnostic bound}) is $\frac{\delta^2}{32} = 0.0003$, while our improved error exponent is  $\min\{\frac{\delta}{4}, \infdiv{\Pi}{Q}\} = 0.025$.
Figure \ref{fig:error exponent} shows that the empirical error exponent of the PAC error probability indeed converges to $\infdiv{\Pi}{Q}$.

\begin{figure}[htbp]
\centering
\includegraphics[width=0.4\textwidth]{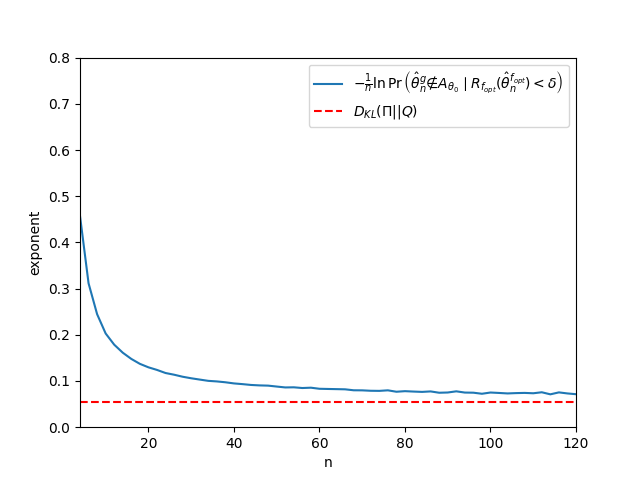}
\caption{Empirical error exponent of the second term in theorem \ref{theorem: realizable+agnostic}. $\ell=2$, $\delta=0.1$. The empirical exponent (blue) was computed using simulation.}
\label{fig:error exponent}
\end{figure}

\section{Conclusions And Future Research}
\label{sec:conclusion}
We derived an improved error exponent for agnostic PAC learning and showed that in some cases agnostic learning might be no harder than realizable learning. Any new realizable learning bound can be plugged into Theorem \ref{theorem: error exponent} to get a better agnostic bound. This result opens new directions for research. One important goal can be to
find explicit conditions for practical hypotheses classes (e.g, neural networks) satisfying the conditions for Theorem \ref{theorem: error exponent}.

Interestingly, the error exponent analysis of PAC learning turns out to be useful in attaining the first theoretical results for the knowledge distillation problem, \cite{adithesis}, providing conditions that define where the associated teacher-student learning is useful and where it is not.
\enlargethispage{-1.2cm} 

\bibliographystyle{IEEEtran}
\bibliography{bibliofile}
\newpage
\
\newpage
\appendices
\section{Proof $A_{\theta}$ form a complete partitioning of $\Theta$}\label{appendix perlimineries}
\begin{lemma}
Let there be a complete hypothesis set \(\Ftheta\) (as in assumption \ref{assumption complete} in the paper), a ground truth function $g$ and a set of GLP's $\Theta_{opt}$ of $R_g(\theta)$. The regions $A_{\theta^*}, \ \theta^* \in \Theta_{opt}$, form a complete partitioning of $\Theta$.
\end{lemma}

\begin{proof}
By definition, the regions are disjoint. Assume exists a distinct set $\Tilde{\Theta}$ of hypotheses that don't belong to any $A_{\theta^*}$, $\theta^* \in \Theta_{opt}$. We will first prove that either all hypotheses in $\Tilde{\Theta}$ coincide or there exist $\Tilde{\theta} \in \Tilde{\Theta}$ such that $\forall \ \theta \in \Tilde{\Theta}\setminus\Tilde{\theta}$ exists a set $\Tilde{X} \subseteq \X$ with positive probability such that $\forall x \in \Tilde{X}, \ \ell(f_\theta(x),g(x)) > \ell(f_{\Tilde{\theta}}(x),g(x))$. Assume by contradiction this is not true, thus exists $\Tilde{\theta}_1 \in \Tilde{\Theta}$ such that $\ell(g(x),f_{\Tilde{\theta}_1}(x)) \leq \ell(g(x),f_{\Tilde{\theta}}(x))$ w.p 1. And for $\Tilde{\theta}_1$ we can find $\Tilde{\theta}_2$ such that $\ell(g(x),f_{\Tilde{\theta}_2}(x)) \leq \ell(g(x),f_{\Tilde{\theta}_1}(x))$ w.p 1. By repeating this we get a series $\Tilde{\theta}_m \in \Tilde{\Theta}$ such that $\ell(g(x),f_{\Tilde{\theta}_{m+1}}(x)) \leq \ell(g(x),f_{\Tilde{\theta}_m}(x))$ w.p 1.\\
If this series is finite, either the series is not distinct and the hypotheses in $\Tilde{\Theta}$ coincide or for any other hypothesis in $\Tilde{\Theta}$ we can find a set with positive probability for which the last element in the series has lower loss, which is a contradiction. \\
If the series is infinite, then the series $\int_{\X}\ell(g(x),f_{\Tilde{\theta}_m}(x))d\mathcal{F}(x)$ has a limit because it is monotone. By the completeness assumption, this means that the limit hypothesis of $f_{\theta_m}$ is in $F_{\Theta}$. If it is also in $F_{\Tilde{\Theta}}$, then either all hypotheses in $\Tilde{\Theta}$ coincide or for any other hypothesis in $\Tilde{\Theta}$ we can find a set with positive probability for which the limit hypothesis has a lower loss, which is a contradiction. If the limit hypothesis is not in $F_{\Tilde{\Theta}}$, then the whole series $\theta_m$ belongs to one of $A_{\theta^*}$, $\theta^* \in \Theta_{opt}$, which is a contradiction. 

To conclude this part, we've showed that either all hypotheses in $\Tilde{\Theta}$ coincide to a single hypothesis $\Tilde{\theta}$ or there exist $\Tilde{\theta} \in \Tilde{\Theta}$ for which $\forall \ \theta \in \Tilde{\Theta}\setminus\Tilde{\theta}$ exists a set $\Tilde{X} \subseteq \X$ with positive probability such that $\forall x \in \Tilde{X}, \ \ell(f_\theta(x),g(x)) > \ell(f_{\Tilde{\theta}}(x),g(x))$. 
\\
\\
Because $\Tilde{\theta}$ doesn't belong to any of the regions $A_{\theta^*}$, $\theta^* \in \Theta_{opt}$, it means $\Tilde{\theta}$ is not a GLP, as every GLP belongs to its own $A_{\theta^*}$ region. It also means that for every GLP $\theta^* \in \Theta_{opt}$ exists a set $\Tilde{X} \subseteq \X$ with positive probability such that \(\ell(f_{\Tilde{\theta}}(x),g(x)) < \ell(f_{\theta^*}(x),g(x)) \ \forall x \in \Tilde{X}\). By definition of $A_{\theta^*}$, we have
$\Pro{\ell\big(f_{\theta^*}(x), g(x)\big) \leq \ell\big(f_{\theta}(x), g(x)\big)}=1 \ \forall \theta \in A_{\theta^*}$.\\
Thus, $ \forall \theta \in \cup_{\theta^* \in \Theta_{opt}} A_{\theta^*}$ there exists a set $\Tilde{X} \subseteq \X$ with positive probability such that \(\ell(f_{\Tilde{\theta}}(x),g(x)) < \ell(f_{\theta}(x),g(x)) \ \forall x \in \Tilde{X}\). We got that $\forall \theta \in \Theta\setminus\Tilde{\theta}$ exist a set $\Tilde{X} \subseteq \X$ with positive probability such that \(\ell(f_{\Tilde{\theta}}(x),g(x)) < \ell(f_{\theta}(x),g(x)) \ \forall x \in \Tilde{X}\), which is a contradiction to $\Tilde{\theta}$ not being a GLP.\\
Thus, the set $\Tilde{\Theta}$ is empty and the regions $A_{\theta^*}, \ \theta^* \in \Theta_{opt}$ form a complete partitioning of $\Theta$. 
\end{proof}

\section{Supplemented material for "Theoretical Results" section}\label{appendix: PAC analysis}
In this section we'll prove theorems \ref{theorem: realizable+agnostic} and \ref{theorem: error exponent}. This will be done gradually: in subsection \ref{appendix: equivalence} we'll prove an equivalent formulation to the PAC error probability that will be used in proving the theorems. In subsection \ref{appendix: proof 1}, we'll prove theorem \ref{theorem: realizable+agnostic} along with 2 needed Lemmas. In section \ref{appendix: bounds} we'll prove a Lemma about lower and upper bounds on the seconds term in theorem \ref{theorem: realizable+agnostic} that will be used in proving theorem \ref{theorem: error exponent}. Finally, in \ref{appendix: proof theorem 2} we'll prove theorem \ref{theorem: error exponent} along with with 3 needed Lemmas. The proofs of theorems \ref{theorem: realizable+agnostic} and \ref{theorem: error exponent} in this section are provided for the case of a finite number $K+1$ of GLP's. This is generalized to an infinite number of GLP's in section \ref{appendix:generalization} of the appendix. Note that we will sometimes refer to equations from the paper.
\subsection{K-boundary has stable global GLP}\label{appendix: K-boundary} 
\begin{lemma}\label{lemma:K-boundary}
Let $g(x)$, $x \in [0,1]$ be a binary ground truth function with at most $M$ transition points between $0$ and $1$. Let $\Ftheta$ be the K-boundary class. Then the optimal hypothesis when learning from $g(x)$ with binary loss is a stable GLP.
\end{lemma}
\begin{proof}
    Let $\theta_0$ be the optimal hypothesis when learning from $g$ with binary loss. $\theta_0$ is a GLP as for any other hypothesis in the class there must exist a set in $\X$ for which $\theta_0$ is uniformly better (otherwise it wouldn't minimize the risk).Denote the transition point of $g$ as $b_1,...,b_M$ and the transition points of $f_{\theta_0}$ as $a_1,...,a_K$. Assume without loss of generality that $b_i$ isn't equal to either $0$ or $1$ and denote $b_0=0$ and $b_{M+1}=1$. First, we'll prove that every transition point of $f_{\theta_0}$ equals to one of $b_0,...,b_{M+1}$ - assume by contradiction exists $a_i$ that isn't equal to one of $b_0,...,b_{M+1}$, thus satisfying $b_j < a_i < b_{j+1}$, where $j$ is one of $0,..., M+1$. For $b_j<x<b_{j+1}$, $g(x)$ is constant (either 0 or 1). We can generate 2 new hypotheses by changing $a_i$ to $b_j$ or to $b_{j+1}$, at least of these new hypotheses has zero loss for $x \in [b_j, b_{j+1}]$ and coincides with $\theta_0$ outside of it. Notice that $\theta_0$ has non-zero loss on $[b_j, b_{j+1}]$ as it can coincide with $g$ either on $[a_i, b_{j+1}]$ or on $[b_j, a_i]$. This, $\theta_0$ is not the risk minimizer, which is a contradiction.\\
    We move to prove that for every $[b_j, b_{j+1}]$, if exist $i$ such that either $a_i=b_j$ or $a_i=b_{j+1}$, then $f_{\theta_0}(x) = g(x)$ for $x \in[b_j, b_{j+1}]$ - assume by contradiction that there is such subset $[b_j, b_{j+1}]$ for which $a_i = b_j$ (the case of $a_i = b_{j+1}$ is analogous ) and $f_{\theta_0}(x)$ doesn't coincide with $g(x)$. Because both $g(x)$ and $f_{\theta_0}(x)$ don't have any additional transition points between $b_j$ and $b_{j+1}$, they are constant in this region and $f_{\theta_0} \neq g(x)$ for $b_j<x<b_{j+1}$. By changing $a_i$ to be $b_{j+1}$ instead of $b_j$, we obtained a new K-boundary hypothesis that coincides with $f_{\theta_0}$ outside $[b_j, b_{j+1}]$ and coincides with $g$ on $[b_j, b_{j+1}]$ thus obtaining better risk than $\theta_0$ which is a contradiction to its optimality. \\
    Denote $\epsilon = \min_{i\in\{ 0,..,M\}}||b_{i+1} - b_i ||/2 > 0$, and some perturbation $\Tilde{\theta}_0$ with transition points $a'_1, ..., a'_K$, where $||(a'_1, ..., a'_K) - (a_1, ..., a_K)|| < \epsilon$.
    For every $i =1,..,K$ we know that $a_i$ is one of $b_j$. thus, if $a_i =b_j$, then \\
    $b_j -\epsilon <a'_i < b_j +\epsilon$
    Thus, $b_{j-1} < a'_i < b_{j+1}$. To show $\theta_0$ is stable, we need to show that any region for which $f_{\theta_0}(x) = g(x)$, we also have $f_{\Tilde{\theta}_0}(x) = g(x)$. For every $j=0,..,M$, $f_{\theta_0}(x)$ might have loss on the region $[b_j, b_{j+1}]$ only if there is no $i$ for which $a_i = b_j$ or $a_i = b_{j+1}$. But this means that for every $i=1,..,K$, if $a_i =b_l$ (where $l \neq j$) then $b_{l-1} < a'_i < b_{l+1}$, thus $a'_i \notin [b_j, b_{j+1}]$. Moreover, if $a_i < b_j$ then $a'_i < b_j$, and if $a_i > b_{j+1}$ then $a'_i > b_{j+1}$. Thus, the amount of transition points of $f_{\theta_0}$ before $b_j$ is the same as the amount of transition points of $f_{\Tilde{\theta}_0}$ before $b_j$. We conclude that $f_{\theta_0}$ and $f_{\Tilde{\theta}_0}$ coincide on any such region $[b_j, b_{j+1}]$. Thus, any region that is missclassified by $f_{\theta_0}$ is also missclassified by $f_{\Tilde{\theta}_0}$ and $\theta_0$ is a stable GLP.
    
\end{proof}

\subsection{Equivalence Lemma}\label{appendix: equivalence}
The following Lemma shows the equivalence:
\begin{align*}
    \Pro{R_g(\gemp) - R_g(\theta_{opt})> \delta} = \Pro{R_{f_{opt}}(\gemp)>\delta}
\end{align*}
This will allow us to use the simpler right hand term instead of the PAC error probability.

\begin{lemma}\label{theorem:equivalence}
Let there be a hypothesis class $\Ftheta$ and a ground truth function $g(x)$ with projection $\fopt$ on the hypothesis class. Under assumptions \ref{assumption stable}-\ref{assumption delta} from the main paper,
for every $0<\delta < \delta_{max}$ and $\theta \in \Theta$ the following holds:
\begin{align}
    R_g(\theta) - R_g(\theta_{opt}) < \delta \Longleftrightarrow \Ropt(\theta) <\delta , \ \theta \in \Theta
\end{align}
\end{lemma}

\begin{proof}
$\\$
$\underline{\Longrightarrow}$: 
We have the following due to $\delta < \delta_{max}$:
\begin{align*}
    \Ropt(\theta)<\delta<\min_{\Tilde{\theta} \notin A_{\theta_0}}\Ropt(\Tilde{\theta}) \Longrightarrow \theta \in A_{\theta_0}
\end{align*}
For $\theta \in A_{\theta_0}$ we have:  
\begin{align*}
    \ell(g(x), \fopt(x))=1\ \Longrightarrow\ \ell(g(x), f_{\theta}(x))=1
\end{align*}
Denote:
\begin{align*}
    X_1=\{x:\ell(g(x),\fopt(x))=1\} \\
    X_0 = \{x:\ell(g(x),\fopt(x))=0\}
\end{align*}
The following chain of equalities holds for $\theta \in A_{\theta_{0}}$:
\begin{align*}
    &R_{g}(\theta) = \int\limits_{X}\ell(g(x),f_{\theta}(x))d\mathcal{F}(x) \\
    &= \int\limits_{X_1}\ell(g(x),f_{\theta}(x))d\mathcal{F}(x) + \int\limits_{X_0}\ell(g(x),f_{\theta}(x))d\mathcal{F}(x)\\
    &= \int\limits_{X_1}1d\mathcal{F}(x) + \int\limits_{X_0}\ell(\fopt(x),f_{\theta}(x))d\mathcal{F}(x) \\
    &= R_{g}(\topt) + \int\limits_{X}\ell(\fopt(x),f_{\theta}(x))d\mathcal{F}(x) \\
    &- \int\limits_{X_1}\ell(\fopt(x),f_{\theta}(x))d\mathcal{F}(x) \\
    &= R_{g}(\topt) + \Ropt(\theta)
\end{align*}
So, for \(\theta \in A_{\theta_0}\) we have $R_{g}(\theta) - R_g(\theta_{opt}) = \Ropt(\theta)$.
We got
$R_g(\theta) - R_g(\theta_{opt}) < \delta$.
\\
$\underline{\Longleftarrow}$: 
We have the following due to $\delta < \delta_{max}$:
\begin{align*}
    &R_{g}(\theta) - R_{g}(\topt) < \delta < 
    \min_{\theta \notin A_{\theta_0}}R_{g}(\theta) 
    - R_{g}(\topt)
\end{align*}
Thus, $\theta \in A_{\theta_0}$. We've already showed that for $\theta \in A_{\theta_0}$ we have $R_{g}(\theta) - R_{g}(\topt) = \Ropt(\theta)$. 
So, we got $\Ropt(\theta) < \delta$
\end{proof} 

\subsection{Proof of Theorem \ref{theorem: realizable+agnostic}}\label{appendix: proof 1}
In this subsection we will prove theorem \ref{theorem: realizable+agnostic} from the main paper.
Before proving it, we first need to prove 2 lemmas that will be used as part of the proof. The proof of theorem \ref{theorem: realizable+agnostic} is given in the end of this subsection.
\begin{lemma}\label{lem_cond1}
Given hypothesis set \(\Ftheta\), ground truth function $g$ with projection $\fopt$, and a drawn sample $x^n$, the following holds under assumptions \ref{assumption stable}-\ref{assumption delta} from the main paper:
\begin{align}
    \Pro{\Ropt(\hat{\theta}^{g}(x^n)) >\delta  \ |  \Ropt(\hat{\theta}^{\fopt}(x^n)) > \delta}=1 
\end{align}
\end{lemma}
\begin{proof}
From assumption \ref{assumption delta} we have $\delta <\delta_{max} \leq \min_{\theta \notin A^g_{\theta_{0}}}\Ropt(\theta)$. Thus, if $\hat{\theta}^{g}(x^n) \notin A_{\theta_0}$, then we have $\Ropt(\hat{\theta}^{g}(x^n)) >\delta$ and we are done. \\
Let's focus on the case $\hat{\theta}^{g}(x^n) \in A_{\theta_0}$.
Denote:
\begin{align}
    \Tilde{x}^k = \{x \in x^n \mid g(x) \neq \fopt(x)\}
\end{align}
We have the following:
\begin{equation}
\begin{aligned}\label{equality on tilde_x_k} 
    &\gemp \in A_{\theta_0} \Longrightarrow \ell(f_{\gemp}(x), g(x)) \ge \ell(\fopt(x), g(x)) \ \text{w.p 1}\\
    &\Longrightarrow f_{\hat{\theta}_n^{g}}(x) = \fopt(x)\ \forall \ x \in \Tilde{x}^k
\end{aligned}
\end{equation}
The empirical risks for any $\theta$ can be decomposed:
\begin{equation*}
\begin{aligned}
R^{emp}_{\fopt}(\theta,x^n) =&\frac{1}{n} \sum_{x \in \Tilde{x}^k}\ell(\fopt(x),f_\theta(x)) +\\
    &\frac{1}{n} \sum_{x \in x^n \setminus\Tilde{x}^k}\ell(\fopt(x),f_\theta(x))\\
    R^{emp}_{g}(\theta,x^n) =&\frac{1}{n} \sum_{x \in \Tilde{x}^k}\ell(g(x),f_\theta(x)) +\\
    &\frac{1}{n} \sum_{x \in x^n \setminus\Tilde{x}^k}\ell(g(x),f_\theta(x))
    \end{aligned}
\end{equation*}
For $\theta \in A_{\theta_0}$, the empirical risk with regard to $g$ is:
\begin{align}
     R^{emp}_{g}(\theta,x^n) = \frac{k}{n} + \frac{1}{n} \sum_{x \in x^n \setminus\Tilde{x}^k}\ell(\fopt(x),f_\theta(x))
\end{align}
Thus, $R^{emp}_{g}(\theta,x^n)$ can be decomposed to 2 terms -  a fixed term and a term that is minimized by $\fopt$. Thus, $\topt$ is the minimizer of $R^{emp}_{g}(\theta,x^n)$, so the ERM $f_{\gemp}$ will choose a hypothesis that is equal to $\fopt$ on the set $x^n \setminus\Tilde{x}^k$. From equation (\ref{equality on tilde_x_k}), $\fopt$ and $f_{\hat{\theta}_n^{g}}$ are also equal on $\Tilde{x}^k$, thus they are equal on the entire sample $x^n$.
We also have $R^{emp}_{\fopt}(\optemp,x^n ) = 0$ because the empirical risk is zero in realizable learning, so $\fopt$ and $f_{\optemp}$ are equal on the entire sample $x^n$. We get:
\begin{align}
    f_{\gemp}(x) = \fopt(x) = f_{\optemp} \ \forall \ x \in x^n
\end{align}
$\gemp$ and $\optemp$ have the same empirical risk. By the convention the ERM is the hypothesis with minimum empirical risk that maximizes the true risk, which is equivalent to $\Ropt(\theta)$ (recall the equivalence in appendix section \ref{appendix: equivalence}). Thus, because they have the same empirical risk, $\gemp$ and $\optemp$ are equal and we have:
\begin{align*}
    &\gemp \in A_{\theta_0} \Longrightarrow f_{\gemp} = f_{\optemp}\\ &\Longrightarrow \Ropt(\gemp) = \Ropt(\optemp) >\delta
\end{align*}
To conclude, the following holds for $\Ropt(\gemp) > \delta$:
\begin{align*}
    \Pro{\Ropt(\hat{\theta}^{g}(x^n)) >\delta  \mid \Ropt(\hat{\theta}^{\fopt}(x^n)) > \delta}=1
\end{align*}
\end{proof}

\begin{lemma}\label{lem_cond2}
Given hypothesis class $\Ftheta$, ground truth $g$ with projection $\fopt$ on the class and a drawn sample $x^n$, the following holds under assumptions \ref{assumption stable}-\ref{assumption delta} from the main paper:
\begin{equation}
    \begin{aligned}
    &\Pro{ \Ropt(\hat{\theta}^{g}(x^n)) >\delta  \ | \ 
    \gemp \in A_{\theta_0},\ \Ropt(\hat{\theta}^{\fopt}(x^n)) < \delta}\\&=0 
\end{aligned}
\end{equation}
\end{lemma}

\begin{proof}
Given $\Ropt(\hat{\theta}^{\fopt}(x^n)) < \delta$, any hypothesis $\theta^{'}$ with $\Ropt(\theta^{'}) > \delta$ doesn't achieve minimal empirical risk on $x^n$ with regard to $\fopt$ . This is true due to the convention that the ERM is the hypothesis with maximal risk from all the hypotheses with minimal empirical risk and due to the equivalence in section \ref{appendix: equivalence}. Denote: 
\begin{align}
    \Tilde{x}^k = \{x \in x^n \mid g(x) \neq \fopt(x)\} \\
    \hat{x}^m = \{x \in x^n \mid f_{\theta^{'}}(x) \neq \fopt(x)\}
\end{align}
Let's assume $\gemp = \theta^{'}$ , which means $\Ropt(\gemp) > \delta$. \\
We are given that $\gemp \in A_{\theta_0}$, so we have $\theta^{'} \in A_{\theta_0}$, which means by definition of $A_{\theta_0}$, that if $f_{\theta^{'}}(x) = g(x)$ then $\fopt(x) = g(x)$. Thus $\Tilde{x}^k \cap \hat{x}^m = \emptyset$.
\\ 
Notice that $\hat{x}^m \neq \emptyset$ because $\topt$ achieves lower empirical risk with regard to $\fopt$ than $\theta^{'}$, so there must be at least one sample of $x^n$ on which $\topt$ is better, otherwise 
$f_{\theta^{'}}$ and $\fopt$ coincide on $x^n$
 and have the same empirical risk, which is a contradiction to $\theta^{'}$ not being ERM with regard to $\fopt$. The empirical risk can be decomposed to:
\begin{equation}
    \begin{aligned}
        &R^{emp}_{g}(\theta,x^n) =\frac{1}{n} \sum_{x \in \hat{x}^m}\ell(g(x),f_\theta(x)) \\
        &+\frac{1}{n} \sum_{x \in  \Tilde{x}^k}\ell(g(x),f_\theta(x))
    \end{aligned}
\end{equation}

By definition of $\hat{x}^m$, we have 
$ R_g^{emp}(\topt,\hat{x}^m) < R_g^{emp}(\theta^{'},\hat{x}^m)$.\\
Form $\theta^{'} \in A_{\theta_{0}}$ we get $R_g^{emp}(\topt, \Tilde{x}^k) = R_g^{emp}(\theta^{'},\Tilde{x}^k)$.\\
Thus we have $R_g^{emp}(\topt,x^n) < R_g^{emp}(\theta^{'},x^n)$.\\
We got that $\topt$ achieves lower empirical risk with regard to $g$, which is a contradiction to $\gemp=\theta^{'}$. Thus, if $\gemp\neq\theta^{'}$ then :\\ $\Pro{ \Ropt(\hat{\theta}^{g}(x^n)) >\delta \mid
    \gemp \in A_{\theta_0},\Ropt(\hat{\theta}^{\fopt}(x^n)) < \delta}=0$.
\end{proof}
\textbf{Proof of theorem \ref{theorem: realizable+agnostic}:}\\
By conditioning $\Pro{\Ropt(\gemp) < \delta}$ on $\Ropt(\optemp)$ we get:
\begin{align*}
    &\Pro{\Ropt(\gemp) < \delta} = \\
    &\Pro{\Ropt(\gemp)<\delta \mid \Ropt(\optemp)<\delta}
    \Pro{\Ropt(\optemp)<\delta}\\  
    &+\Pro{\Ropt(\gemp) < \delta \mid \Ropt(\optemp) > \delta}\\
    &\quad\cdot\Pro{\Ropt(\optemp) > \delta}=\\
    & \Pro{\Ropt(\gemp) < \delta \mid \Ropt(\optemp) < \delta}
    \Pro{\Ropt(\optemp) < \delta}
\end{align*}
Where the last equality is due to Lemma \ref{lem_cond1}.
By conditioning on $\{\gemp \in A_{\theta_0}\}$, we get the following:
\begin{align*}
    &\Pro{\Ropt(\gemp) < \delta \mid \Ropt(\optemp) < \delta}= \\
    &\Pro{\Ropt(\gemp) < \delta \mid \Ropt(\optemp) < \delta,\ \gemp \in A_{\theta_0}}\\
    &\cdot \Pro{\gemp \in A_{\theta_0} \mid \Ropt(\optemp) < \delta}\\
    & +\Pro{\Ropt(\gemp) < \delta \mid \Ropt(\optemp) < \delta,\ \gemp \notin A_{\theta_0}}\\
    &\cdot\Pro{\gemp \notin A_{\theta_0} \mid \Ropt(\optemp)< \delta}\\ 
    &= \Pro{\gemp \in A_{\theta_0} \mid \Ropt(\optemp) < \delta}
\end{align*}
Where the last equality is due to  Lemma \ref{lem_cond2} and because for $\delta<\delta_{max}$ we have $\gemp \notin A_{\theta_0} \Longrightarrow \Ropt(\gemp) > \delta$. We conclude with the following equality:
\begin{align*}
    &\Pro{\Ropt(\gemp) < \delta} =\\ &\Pro{\gemp \in A_{\theta_0} \mid \Ropt(\optemp) < \delta} \Pro{\Ropt(\optemp) < \delta}
\end{align*}
By denoting $\Pro{\Ropt(\optemp) < \delta} = 1-P_R$ and taking the complement probability, we get: 
\begin{align*}
    &\Pro{\Ropt(\gemp) > \delta} = \\
    &P_R +(1-P_R)\Pro{\gemp \notin A_{\theta_0} \mid \Ropt(\optemp) < \delta}
\end{align*}

\subsection{Bounds Lemma}\label{appendix: bounds}
\begin{lemma}\label{lemma: bounds}
    Under assumptions \ref{assumption stable}-\ref{assumption delta} from the main paper, there exists a number $\ell \in \mathbf{N}$ such that the following holds:
\begin{equation*}
\begin{aligned}
    &\Pro{\gemp \notin A_{\theta_0} \mid \Ropt(\optemp) < \delta} \\
    &\ge 
    \Pro{\cup_{i=1}^K \{\#D_i + \ell \leq \#D_i^{'}\}_{n-\ell}}\\
    &\Pro{\gemp \notin A_{\theta_0} \mid \Ropt(\optemp) < \delta}\\ &\leq \Pro{\cup_{i=1}^K \{\#D_i \leq \#D_i^{'}\}_n}
\end{aligned}
\end{equation*}
\end{lemma}
Before the proof, let's denote the following concept of the set of minimal sequences:
\begin{definition}
Let there be a hypothesis class $F_{\Theta}$, a function $\fopt \in F_{\Theta}$, and a number $0<\delta<\delta_{max}$. The set of minimal sequences with risk lower than $\delta$ with regard to $\fopt$ is
\begin{equation*}\label{def:X_min}
\begin{aligned}
    X_{min}^{\delta} = \{\vec{x} \mid \Ropt(\optemp(\vec{x})) < \delta, \Ropt(\optemp(\vec{x}/x_i)) > \delta \ \forall \ i \}
\end{aligned}
\end{equation*}
Where $\vec{x}/x_i$ is $\vec{x}$ without the i'th component.
\end{definition}
This set is nonempty due to assumption \ref{assumption consistent}.
Denote $\ell$ as the maximal length of a sequence in $X_{min}^{\delta}$. For example, in the k-boundary hypothesis class we have $\ell \leq 2k$ for any hypothesis.
Thus, any i.i.d sequence $x^n$ achieving $\Ropt(\optemp(\vec{x})) < \delta$ can be decomposed into a minimal sequence of length at most $\ell$ and the rest of the samples which have no constraint on them.
So, if $\Ropt(\optemp(\vec{x})) < \delta$, then there is a constraint on at most $\ell$ samples of $\vec{x}$.
We'll now state the proof for Lemma \ref{lemma: bounds}.
\begin{proof}
    we have $\Pro{\gemp \notin A_{\theta_0} \mid \Ropt(\optemp) < \delta} = \Pro{\cup_{i=1}^K \{\#D_i \leq \#D_i^{'}\}_n \mid \Ropt(\optemp) < \delta}$. Let the maximum length of a set in $X_{min}^{\delta}$ be $\ell$. So any sequence $x^n$ that satisfies $\Ropt(\gemp(x^n))<\delta$ can be decomposed to a minimal sequence of length at most $\ell$ and the rest of the samples:
\begin{align*}
    &\Pro{\gemp \notin A_{\theta_0} \mid \Ropt(\optemp) < \delta} = \\
    &\Pro{\cup_{i=1}^K \{\#D_i \leq \#D_i^{'}\}_n \mid \text{constraint on at most $\ell$ samples}}
\end{align*}
The lower bound is obtained by assuming that all $\ell$ samples fell in every $D_i$ region for $i=1,...K$:
\begin{align*}
    &\Pro{\gemp \notin A_{\theta_0} \mid \Ropt(\optemp) < \delta} \ge \\
    &\Pro{\cup_{i=1}^K \{\#D_i + \ell \leq \#D_i^{'}\}_{n-\ell}}
\end{align*}
The upper bound is obtained by assuming that all $\ell$ samples didn't fall in any all $D_i$ region for $i=1,...K$ (i.e., they all fell in $X_c$):
\begin{align*}
    &\Pro{\gemp \notin A_{\theta_0} \mid \Ropt(\optemp) < \delta} \leq \\ &\Pro{\cup_{i=1}^K \{\#D_i \leq \#D_i^{'}\}_n}
\end{align*}
\end{proof}

\subsection{Proof of theorem \ref{theorem: error exponent}}\label{appendix: proof theorem 2}
In this subsection we'll prove theorem \ref{theorem: error exponent}. This will be done by showing that the error exponent of the bounds in Lemma \ref{lemma: bounds} is $\infdiv{\Pi}{Q}$. First, we'll analyze the second term in theorem \ref{theorem: realizable+agnostic}.  Using Eq.(\ref{D_pairs}), we have:\\
\begin{equation}
    \begin{aligned}
    &\{\hat{\theta}^g(x^n) \notin A_{\theta_0}\}
    = \cup_{i=1}^K \{R_g^{emp}(\theta_0,x^n) \ge R_g^{emp}(\theta_i,x^n)\}
    \\&= \cup_{i=1}^K \{\#D_i \leq \#D_i^{'}\}_n
    \end{aligned}
\end{equation}
Subscript $n$ indicates the length of the sample $x^n$. Thus we have:
\begin{equation}
    \begin{aligned}
        &\Pro{\gemp \notin A_{\theta_0} \mid \Ropt(\optemp) < \delta} = \\
    &\Pro{\cup_{i=1}^K \{\#D_i \leq \#D_i^{'}\}_n \mid \Ropt(\optemp) < \delta}
    \end{aligned}
\end{equation}
Denote the vector of non-negative integers \\$\Vec{m}=(m_1,..,m_{1,..,k},m^{'}_1,..,m^{'}_{1,..,k},m_c)$. Using Eq.(\ref{X_region}), we have the following:
\begin{equation}\label{sum of type}
\begin{aligned}
    &\Pro{\cup_{i=1}^K \{\#D_i \leq \#D_i^{'}\}_n} = \sum_{\Vec{m}\in M_n} Pr\big(\#X_1=m_1,..,\\
    &\#X_{1,..,k} = m_{1,..,k}, \#X^{'}_1=m^{'}_1,..,\#X^{'}_{1,..,k} = m^{'}_{1,..,k},\\
    &\#X_c=m_c\big)
\end{aligned}
\end{equation}
Where $M_n$ is the set of integers with sum $n$ that satisfy at least one of $\{\#D_i \leq \#D_i^{'}\}_n$ : 
\begin{equation}
    \begin{aligned}
        &M_n = \Bigg\{\Vec{m} \mathrel{\Bigg|}  \Big\{ A\begin{pmatrix}
        m_1,
        ..,
        m_{1,..,K},
        ..,
        m^{'}_{1,..,K}
        \end{pmatrix}^t < 0\Big\}^c ,  \\ &m_1+..+m^{'}_1+..+ m_c = n
        \Bigg\}
    \end{aligned}
\end{equation}
$A$ is the matrix from Eq.(\ref{eq: constraint matrix}).
The type of a sequence $x^n$ on alphabet $\chi$ is the empirical distribution of symbols in the sequence:
\begin{align*}
    P_{x^n} = (\frac{\#a_1}{n},\frac{\#a_2}{n},...,\frac{\#a_r}{n}), \quad a_1,...a_r \in \chi
\end{align*}
Denote $\mathcal{P}_n$ as the set of all length $n$ sequences types and $T(P)$ as the set of sequence $x^n$ with type $P$.
Our problem can be formulated as an i.i.d sequence over 
 alphabet $\chi$.
$M_n$ can be formulated as a constraint on types instead of integers, denoted as $\Tilde{M}_n$ :
\begin{equation}
    \begin{aligned}\label{M_n Tilde}
    &\Tilde{M}_n = \Biggl\{(\frac{\#a_1}{n},...,\frac{\#a_{|\chi|}}{n}) \mathrel{\Bigg|} \Bigg\{A\begin{pmatrix}
\frac{\#a_1}{n}\\
...\\
\frac{\#a_{|\chi|-1}}{n}
\end{pmatrix}<0\Bigg\}^c ,\\
&\frac{\#a_1}{n}+...+ \frac{\#a_{|\chi|}}{n} = 1
\Biggl\} \ , \ a_i \in \chi
\end{aligned}
\end{equation}
Notice that the sets $\Tilde{M}_n$ are subsets of the set $\Pi$ denoted in Eq.(\ref{eq: pi}).
Thus, Eq. (\ref{sum of type}) is the sum of types of sequences of length $n$ that are contained in $\Pi$:
\begin{align*}
    \Pro{\cup_{j=1}^K\{\#D_j \leq \#D^{'}_j\}_n} = \sum_{P \in \mathcal{P}_n \cap \Pi} \Pro{T(P)}
\end{align*}
Notice $Q$ is not contained in $\Pi$ because of the consistency assumption. 
The same formulation can be done for $\Pro{\cup_{i=1}^K \{\#D_i + \ell \leq \#D_i^{'}\}_{n-\ell}}$.
Denote the set $\Pi_{n,\ell}$:
\begin{equation}
\begin{aligned}\label{eq:pi_n_ell}
    &\Pi_{n,\ell} = \Biggl\{(p_{1},...,p_{{|\chi|}}) \mathrel{\Bigg|} 
    \Biggl\{A\begin{pmatrix}
p_{1}\\
...\\
p_{{|\chi|-1}}
\end{pmatrix} < \frac{\ell}{n-\ell} \Biggl\}^c,\\
 &\sum_{i=1}^{|\chi|}p_i=1, p_i\ge0\Biggl\} 
\end{aligned}    
\end{equation}
We have:
\begin{align*}
    \Pro{\cup_{j=1}^K\{\#D_j + \ell \leq \#D^{'}_j\}_{n-\ell}} = \sum_{P \in \mathcal{P}_{n-\ell} \cap \Pi_{n,\ell}} \Pro{T(P)}
\end{align*}
Theorem 3.3 in \cite{mot} states that if a set of probabilities $\Pi$ on $\chi$, that doesn't contain the underlying distribution $Q$, has the property:
\begin{align*}\label{cond_kldiv}
    \lim_{n\rightarrow\infty} \infdiv{\Pi \cap \mathcal{P}_n}{Q} = \infdiv{\Pi}{Q}
\end{align*}
Then the following holds:
\begin{align*}
    \lim_{n\rightarrow\infty} \frac{1}{n} \log\Pro{T(x^n) \in \Pi} = -\infdiv{\Pi}{Q}
\end{align*}
The 3 Lemmas in the end of this section show this condition is satisfied for both $\Pi$ and $\Pi_{n,\ell}$. we have:
\begin{equation*}\label{error exponent limit}
\begin{aligned}
     &\lim_{n\rightarrow\infty} \frac{1}{n} \log\Pro{T(x^n) \in \Pi} = -\infdiv{\Pi}{Q}
     \\
     &\lim_{n\rightarrow\infty} \frac{1}{n-\ell} \log\Pro{T(x^{n-\ell}) \in \Pi_{n,\ell}} = -\infdiv{\Pi}{Q}
\end{aligned}
\end{equation*}
This means that the upper and lower bounds from Lemma \ref{lemma: bounds} have the same error exponent $\infdiv{\Pi}{Q}$. Thus, the error exponent of $\Pro{\gemp \notin A_{\theta_0} \mid \Ropt(\optemp) < \delta}$ is also  $\infdiv{\Pi}{Q}$. By using theorem \ref{theorem: realizable+agnostic} and the error exponent for the uniform realizable case we get: 
\begin{align*}
    \Pro{R_g(\gemp) - R_g(\theta_{opt})> \delta} \doteq e^{-n\cdot \min\{\frac{\delta}{4}, \infdiv{\Pi}{Q}\}}
\end{align*}
This proves theorem \ref{theorem: error exponent}. The following Lemmas prove the fulfilment of the needed conditions.

\begin{lemma}\label{pi}
Let there be an alphabet $\chi$ with underlying probability Q  and an i.i.d sequence $x^n$ over the alphabet. denote the set $\Pi$ as in Eq.(\ref{eq: pi}):
For $Q \notin \Pi$, the following holds:
\begin{align*}
     \lim_{n\rightarrow\infty} \infdiv{\Pi \cap \mathcal{P}_n}{Q} = \infdiv{\Pi}{Q}
\end{align*}
\end{lemma}
\begin{proof}
$\Pi$ is the outside of a polygon on the probability simplex (including the boundary), thus it is a connected closed space. This means that $\infdiv{\Pi}{Q}$ is achieved for some probability $P^* \in \Pi$, such that $\infdiv{\Pi}{Q} = \infdiv{P^*}{Q}$.
\\
$\infdiv{P}{Q}$ is continuous in $P \in \Pi$, so for every $\epsilon>0$ exists $\delta >0 $ such that if $0<||P-P^* ||<\delta$ then $|\infdiv{P}{Q}-\infdiv{P^*}{Q}|<\epsilon$.
\\
Because $\Pi$ is a closed connected set, Lemma \ref{lem_topology} applies, so for every $\delta > 0$ exists $N$ such that for $n>N$ we have an empirical assignment $\Tilde{P}_n \in \Pi \cap \mathcal{P}_n$ satisfying
\\
$||\Tilde{P}_n - P^*||<\delta$ $\Longrightarrow |\infdiv{\Tilde{P}_n}{Q} - \infdiv{\Pi}{Q}| < \epsilon$ 
\\
$\Longrightarrow |\infdiv{\Pi \cap \mathcal{P}_n}{Q} - \infdiv{\Pi}{Q}| < \epsilon$.
\\
We got that for every $\epsilon >0$ exists $N$ such that for $n>N$ we have 
$|\infdiv{\Pi \cap \mathcal{P}_n}{Q} - \infdiv{\Pi}{Q}| < \epsilon$.
\end{proof}

\begin{lemma}\label{pi_ell}
Let $\chi$ be an alphabet with probability $Q$ and an i.i.d sequence $x^{n-\ell}$, $\ell \in \mathbf{N}$, over $\chi$. Denote the the set $\Pi$ as in Eq.(\ref{eq: pi}) and the set $\Pi_{n,\ell}$ as in Eq.(\ref{eq:pi_n_ell}).  
For $Q \notin \Pi$, the following holds:
\begin{align*}
     \lim_{n\rightarrow\infty} \infdiv{\Pi_{n,\ell} \cap \mathcal{P}_n}{Q} = \infdiv{\Pi}{Q}
\end{align*}
\end{lemma}
\begin{proof}
$\Pi_{n,\ell}$ is the outside of a polygon on $S_{\chi}$ (including the boundary), thus it is a connected closed set. 
We already saw in Lemma \ref{pi} that $\infdiv{\Pi}{Q}$ is achieved for some $P^* \in \Pi$ such that $\infdiv{\Pi}{Q} = \infdiv{P^*}{Q}$.
$\infdiv{P}{Q}$ is continuous in $P \in \Pi$, so for every $\epsilon>0$ exists $\delta >0 $ such that if $0<||P-P^* ||<\delta$ then $|\infdiv{P}{Q}-\infdiv{P^*}{Q}|<\epsilon$.
For every $\frac{\delta}{2} >0$ exists $P^{'} \in \Pi$ satisfying $0<||P^{'}-P^* ||<\frac{\delta}{2}$ such that $P^{'}$ is an interior point of $\Pi$. 
Notice that the boundaries of $\Pi_{n,\ell}$ are converging in $n$ to the boundaries of $\Pi$, so exists $N_1 >0$ such that for $n>N_1$ we have $P^{'} \in \Pi_{n,\ell}$.
\\
 $\Pi_{N_1,\ell}$ is a closed connected set, thus Lemma \ref{lem_topology} applies to it.
So, for every $\frac{\delta}{2} >0$ exists $N_2$ such that for $n > N_2$ we have an empirical assignment $\Tilde{P}_n \in \Pi_{N_1,\ell}$ such that $||\Tilde{P}_n - P^{'}|| < \frac{\delta}{2}$.
Notice that $\Pi_{n_1,\ell} \subset \Pi_{n_2,\ell}$ for $n_1 < n_2$. So, for $n>\max(N_1, N_2)$ we have an empirical assignment $\Tilde{P}_n \in \Pi_{n,\ell} \cap \mathcal{P}_n$ such that$||\Tilde{P}_n - P^{'}|| < \frac{\delta}{2}$, and by the triangle inequality, it satisfies $||\Tilde{P}_n - P^{*}||<\delta$
\\
$\Longrightarrow |\infdiv{\Tilde{P}_n}{Q} - \infdiv{\Pi}{Q}| < \epsilon$ 
\\
$\Longrightarrow |\infdiv{\Pi_{n,\ell} \cap \mathcal{P}_n}{Q} - \infdiv{\Pi}{Q}| < \epsilon$.
\end{proof}

\begin{lemma}\label{lem_topology}
Let $\Pi$ be a closed and connected subset of $\{(p_1,...,p_r) \mid 0\leq p_1 + ... + p_r \leq 1,\  0 \leq p_i \leq 1\}$ and let $\mathcal{P}_n$ be the set of types of sequences of length $n$ over an alphabet $\chi$ of size $r$. For all $P^* \in \Pi$ the following holds:
\begin{align*}
\lim_{n\rightarrow\infty} \inf_{P \in \Pi \cap \mathcal{P}_n}||P-P^*|| = 0  
\end{align*}
\end{lemma}

\begin{proof}
$\Pi$ is a closed set, thus for any $P^*\in \Pi$ and any $\frac{\epsilon}{2}>0$ exists $P_q \in Interior(\Pi)$ such that $||P_q - P^*||< \frac{\epsilon}{2}$ and $P_q$ is rational
    $P_q = (\frac{a_1}{b_1}, ..., \frac{a_r}{b_r}), \ a_i,b_i \in \mathbf{N}$ (because $\mathbf{Q}^r$ is a dense subset of $\Pi$).
For every n denote the empirical probability $\Tilde{P}_n \in \mathcal{P}_n$:
\begin{align*}
    \Tilde{P}_n = (\frac{\floor{\frac{a_1}{b_1}n}}{n}, ...,\frac{\floor{\frac{a_{r-1}}{b_{r-1}}n}}{n},
    \frac{n-\sum_{i=1}^{r-1}\floor{\frac{a_i}{b_i}n}}{n})
\end{align*}
Notice $\Tilde{P}_n$ converges to $P_q$. This means that for any $\frac{\epsilon}{2}>0$ exists $N$ such that for $n>N_1$ we have $||\Tilde{P}_n - P_q|| < \frac{\epsilon}{2}$. 
Because $P_q$ is in the interior of $\Pi$ and $\Tilde{P}_n$ converges to $P_q$, exists $N_2$ such that for $n>N_2$ we have $\Tilde{P}_n \in \Pi$.
Using the triangle inequality, for $n>\max(N1,N2)$ we have  $||\Tilde{P}_n - P^*|| < \epsilon$ and $\Tilde{P}_n \in \Pi \cap \mathcal{P}_n$ $\Longrightarrow \inf_{P \in \Pi \cap \mathcal{P}_n}||P-P^*|| < \epsilon$.\\
This shows that for any $\epsilon >0$ exists $N$ such that for $n>N$ we have $inf_{P \in \mathcal{P}_n}||P-P^*|| < \epsilon$ 
\end{proof}

\section{Generalization to Infinite Amount Of Generalized Optimum Points}\label{appendix:generalization}

In this section we briefly show how to generalize results to the case of infinite amount of GLP's and how to derive $\Pi$ and $Q$. Let the set of GLP's be $\Theta_{opt}$ and $\topt=\theta_0$ the global optimum point. We need the loss of the global optimum $\topt$ to be bounded away from the loss of the other GLP's:
\begin{align}
    \exists\ \epsilon>0 \mid \Ropt(\theta^*)-\Ropt>\epsilon \ \forall \theta^* \in \Theta_{opt}\setminus\theta_{opt}
\end{align}
This is necessary for $\delta_{max}>0$. Notice that this is achieved from assumptions \ref{assumption stable}, \ref{assumption unique} and \ref{assumption complete}.
Due to the completeness of $\Theta$ and uniqueness of $\theta_{opt}$, the only hypotheses $\theta \in \Theta$ that can potentially have a risk that is arbitrarily close to the optimal risk $R_g(\theta_{opt})$ are those that are in the neighborhood $\theta_{opt}$. Due to the stability assumption, we know that exists a small enough neighborhood of $\theta_{opt}$, such that any hypothesis in it will have a higher loss w.p 1.\\
For each GLP $\theta^* \in \Theta_{opt}\setminus\theta_{opt}$ of $R_g(\theta)$, denote:
\begin{equation}
    \begin{aligned}
         D_{\theta^*} = D(\theta_{0},\theta^*), \
         D^{'}_{\theta^*} = D(\theta^*, \theta_{0})
    \end{aligned}
\end{equation}
We have the following: 
\begin{equation*}
\begin{aligned}
    &\Pro{\gemp \notin A_{\theta_0}} = \\
    &\Pro{\bigcup_{\theta^* \in \Theta_{opt} \setminus \theta_0} \{R_g^{emp}(\theta_0,x^n) \ge R_g^{emp}(\theta^*,x^n)\}}= \\
    &\Pro{\bigcup_{\theta^* \in \Theta_{opt} \setminus \theta_0} \{\#D_{\theta^*} \leq \#D_{\theta^*}^{'}\}_n}
\end{aligned}
\end{equation*}
From this point, generalizing the proof of theorem \ref{theorem: realizable+agnostic} is straight forward.
Denote the following regions:
\begin{equation*}
\begin{aligned}
    &X_\phi= disjointify(\{D_{\theta^*}\ , \ \theta^* \in \Theta_{opt} \setminus \theta_0\}) \\
     &X^{'}_\phi= disjointify(\{D^{'}_{\theta^*}\ , \ \theta^* \in \Theta_{opt} \setminus \theta_0\}) \\
    &X_c= X\setminus\{\cup_{\theta^{'} \in \Theta_{opt} \setminus \theta_0} D^{'}_{{\theta^{'}}}\cup_{\theta \in \Theta_{opt} \setminus \theta_0}{D_\theta}\}
\end{aligned}
\end{equation*}
Where the $disjointify$ operator takes a collection of sets and returns disjoint sets indexed by a continuous index $\phi$. We get a continuous alphabet $\chi$. Denote $X_{\Phi} = \cup_{\phi}X_{\phi}\cup_{\phi}X^{'}_{\phi}\cup X_c$ and let $\chi \subseteq \mathbf{R}$ be generated by a bijective mapping $\Psi: X_{\Phi} \longrightarrow \chi$. We can always find such mapping because $X_\Phi$ is a set of non-intersecting sub-sets of $\X\subseteq\mathbf{R^N}$, so the cardinality of $X_\Phi$ is no greater than the cardinality of $\X$ and hence no greater than the cardinality of $\mathbf{R}$. Thus, there exists a subset $\chi$ of $\mathbf{R}$ with the same cardinality of $X_\Phi$, which means there exists a bijective mapping from $X_\Phi$ to $\chi$, and $Q$ is the distribution on $\chi$.
Denote the following sets:
\begin{equation}
    \Phi(D_\theta) = \{S \in X_{\Phi}\ \mid S \subseteq D_\theta\}
\end{equation}
Let $F_n(r),\ r \in \chi$ be the empirical distribution (CDF) on $\chi$ induced by the drawn sequence $x^n$. That is, if $k$ samples from the sequence $x^n$ landed in the region $X_{\phi} \in X_{\Phi}$, then $F_n(\Psi(X_{\phi})) - \lim_{a \to \Psi(X_{\phi})^{-}}F_n(a) = \frac{k}{n}$. Denote the set of all such empirical distribution functions as $F^n_{\chi}$.
Denote: 
\begin{equation}
\begin{aligned}
    \chi_\theta = \{ r \in \chi   \mid    r=\Psi(X_{\phi}),\ X_{\phi} \in \Phi(D_\theta) \}\\
    \chi^{'}_\theta = \{ r \in \chi   \mid    r=\Psi(X_{\phi}),\ X_{\phi} \in \Phi(D^{'}_\theta) \}
\end{aligned}
\end{equation}
These are the sets of values in the alphabet $\chi$ that corresponds to regions in $D_{\theta}$ and $D^{'}_{\theta}$.
Denote the following set of empirical distribution functions:
\begin{equation*}
    \begin{aligned}
        &\Tilde{M}_{n} = \bigg\{F_n \in F^n_{\chi}  \mid \exists \theta \in \Theta\ \text{s.t}\\
        &\int_{r \in \chi^{'}_\theta}(F_n(r) - \lim_{a \to r^{-}}F_n(a)) \ge  \int_{r \in \chi_\theta}(F_nr) - \lim_{a \to r^{-}}F_n(a)) \bigg\}
    \end{aligned}
\end{equation*}

This the parallel of Eq.(\ref{M_n Tilde}). Let $F_\chi$ be the set of all distribution functions on $\chi$. We can now define the set $\Pi$:
\begin{equation}
    \begin{aligned}
        &\Pi = \bigg\{F \in F_{\chi}  \mid \exists \theta \in \Theta \ \text{s.t}\\
        &\int_{r \in \chi^{'}_\theta}(F(r) - \lim_{a \to r^{-}}F(a)) \ge  \int_{r \in \chi_\theta}(F(r) - \lim_{a \to r^{-}}F(a)) \bigg\}
    \end{aligned}
\end{equation}

This is the parallel of Eq.(\ref{eq: pi}).
The results are generalize to continuous alphabet by using the continuous version of Sanov's theorem - Theorem 11 of \cite{sanov1961probability}.
\end{document}